\documentclass{article} % For LaTeX2e
\usepackage{iclr2025_conference,times}

\usepackage{amsmath,amsfonts,bm}
\usepackage{graphicx}

\def\eqref#1{equation~\ref{#1}}
\def\1{\bm{1}}

\DeclareMathAlphabet{\mathsfit}{\encodingdefault}{\sfdefault}{m}{sl}
\SetMathAlphabet{\mathsfit}{bold}{\encodingdefault}{\sfdefault}{bx}{n}

\DeclareMathOperator{\sign}{sign}

\usepackage{amsthm}

\theoremstyle{definition}
\newtheorem{definition}{Definition}[section]
\newtheorem*{remark}{Remark}
\newtheorem{theorem}{Theorem}[section]

\usepackage{algorithm}
\usepackage{algpseudocode}
\usepackage{booktabs}
\usepackage{multicol,multirow}

\usepackage{tikz}
\usepackage{amsmath}
\usepackage{geometry}

\usetikzlibrary{shapes.geometric, positioning}

\usepackage{hyperref}
\usepackage{url}

\title{Model Mimic Attack: Knowledge Distillation for Provably Transferable Adversarial Examples}

\author{Kirill Lukyanov  \\
Research Center for Trusted Artificial Intelligence, \\
Ivannikov Institute for System Programming of the Russian Academy of Sciences, \\
Moscow Institute of Physics and Technology (National Research University) \\
Moscow, Russia \\
\AND
Andrew Perminov, Denis Turdakov \\
Research Center for Trusted Artificial Intelligence,\\
Ivannikov Institute for System Programming of the Russian Academy of Sciences \\
Moscow, Russia \\
\AND
Mikhail Pautov \\
AIRI, \\
Research Center for Trusted Artificial Intelligence,\\
Ivannikov Institute for System Programming of the Russian Academy of Sciences \\
Moscow, Russia \\
}

\iclrfinalcopy % Uncomment for camera-ready version, but NOT for submission.
\begin{document}

\maketitle

\begin{abstract}
The vulnerability of artificial neural networks to adversarial perturbations in the black-box setting is widely studied in the literature. The majority of attack methods to construct these perturbations suffer from an impractically large number of queries required to find an adversarial example. In this work, we focus on knowledge distillation as an approach to conduct transfer-based black-box adversarial attacks and propose an iterative training of the surrogate model on an expanding dataset. This work is the first, to our knowledge, to provide provable guarantees on the success of knowledge distillation-based attack on classification neural networks: we prove that if the student model has enough learning capabilities, the attack on the teacher model is guaranteed to be found within the finite number of distillation iterations.

%We experimentally show that the successful attack can be found at the cost of several training epochs of a student model and a notably smaller number of queries to the target model than of the competitors' methods.

% propose a method to compute adversarial perturbations for a large black-box teacher model using a simple white-box attack against small student models obtained with knowledge distillation. 

% Our method does not require information about the teacher model's architecture and training dataset. We prove that, under several assumptions, our approach is guaranteed to find an adversarial example for the teacher model within the finite number of iterations.  

  %\textcolor{red}{danger zone! here, we speak about competitors}
\end{abstract}

\section{Introduction}
The robustness of deep neural networks to input perturbations is a crucial property to integrate them into various safety-demanding areas of machine learning, such as self-driving cars, medical diagnostics, and finances.  Although neural networks are expected to produce similar outputs for similar inputs, they are long known to be vulnerable to adversarial perturbations [\cite{Szegedy2014intriguing}] -- small, carefully crafted input transformations that do not change the semantics of the input object, but force a model to produce a predefined decision. The majority of methods to study the adversarial robustness of neural networks are aimed at crafting adversarial perturbations which indicate that, in general, the predictions of a neural network are unreliable. 
The most effective and stealthy attacks require access to the model's gradients and are therefore of little practical use on their own [\cite{Goodfellow2014ExplainingAH, madry2017towards, Carlini2016TowardsET}].
However, in real-world scenarios, machine learning models are often deployed as services that are available via APIs. This setting, although poses certain limitations to exploring the robustness of machine learning as a service (MLaaS) models, does not make the computation of adversarial perturbations impossible [\cite{chen2020hopskipjumpattack,andriushchenko2020square,qin2023training,vob2024rusleattack}]. It is possible to compute an adversarial perturbation for the black-box model by either estimating its gradient in the vicinity of the target point \cite{ilyas2018black, bai2020improving} or using random search \cite{andriushchenko2020square} or applying knowledge transfer to obtain an auxiliary model to attack in the white-box setting \cite{li2023making, gubri2022lgv}.  

However, these methods may require a lot of queries to the target model and, in general,  are not guaranteed to find an adversarial example.  In this paper, we focus on the following research question: is it possible to provably compute an adversarial example for a given black-box classification neural network for a finite number of queries? To answer this question, we propose \emph{Model Mimic Attack}, the framework for conducting a black-box model transfer attack through multiple knowledge distillations. 

Knowledge distillation attack methods have been studied extensively in recent years.  It is used, for example, to protect intellectual property: the surrogate model obtained by extracting the knowledge of the source one and then is used to create watermarks that help to link the generated content and determine its origin [\cite{yuan2022attack, lukas2019deep, kim2023margin, pautov2024probabilistically}]. This approach is also used in attacks on black-box models [\cite{li2023making, gubri2022lgv}]. We propose iterative training of a series of surrogate models on an expanding dataset. This approach allows each subsequent surrogate model to better mimic the behavior of the black-box model.

% The peculiarity of copyright application

% is to extract knowledge as efficiently as possible and create 

% \textcolor{red}{bla bla bla?}

\begin{figure}[tb]
\centering
\includegraphics[width=0.9\textwidth]{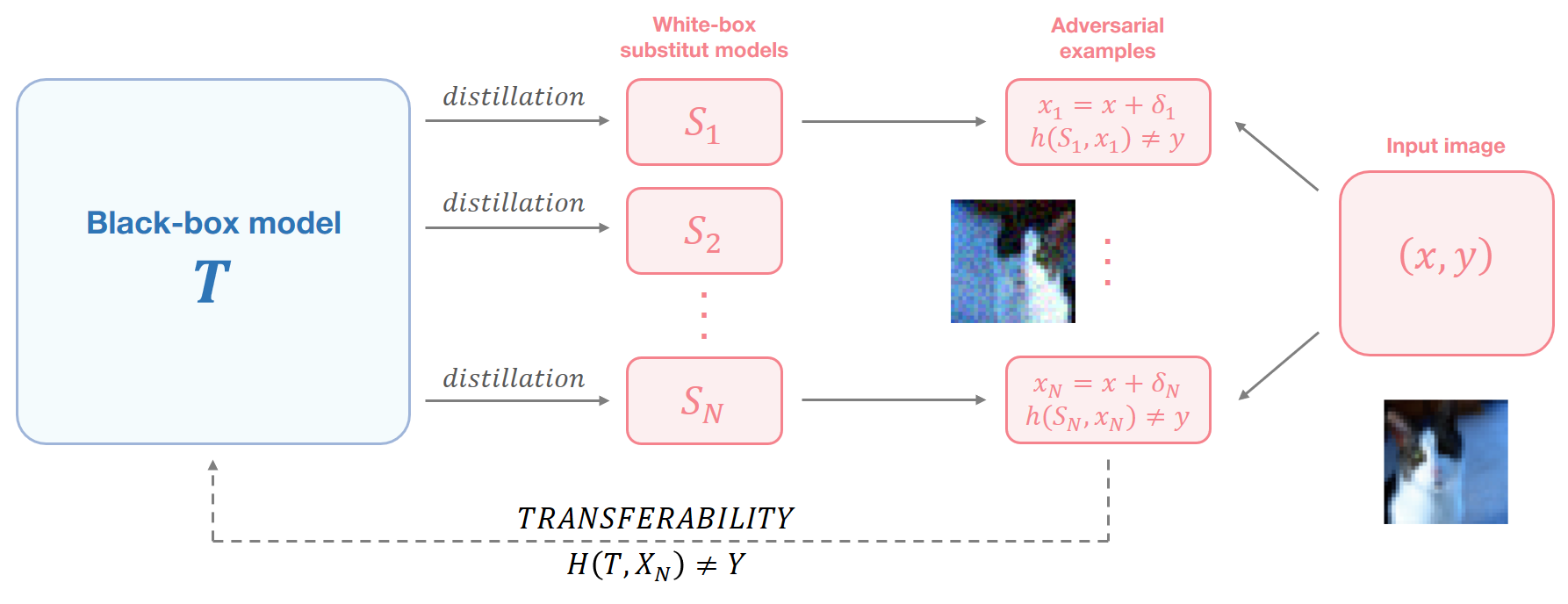}
\caption{The illustration of the proposed method. Given the black-box teacher model $T$, the set of student models $S_1, \dots, S_N$ is obtained via the soft-label knowledge distillation. Each student model is attacked in a white-box manner, and the set of adversarial examples $x_1, \dots, x_N$ is computed. Note that, according to theoretical analysis, there is an adversarial example $x_N$ for the student model $S_N$ which is transferable to the teacher model $T$ for some $N \in \mathbb{N}$. }
\label{fig:method_diagram}
\end{figure}

Our contributions are summarized as follows:
\begin{enumerate}
    \item We propose \emph{Model Mimic Attack}, a score-based black-box model transfer attack via knowledge distillation. The algorithm exploits the behavior of the target teacher network in the vicinity of the target point and yields the set of surrogate student models, which copy the predictions of the target model in the finite set of points. Then, the set of student models is used to compute an adversarial perturbation in the white-box setting, which transfers to the teacher model over a finite number of distillation iterations.
    \item We are the first, to our knowledge, to theoretically show that the distillation-based model transfer attack \emph{is guaranteed} to find an adversarial perturbation for the black-box teacher model. 
    \item We experimentally demonstrate the efficiency of the proposed approach over other black-box attack methods in the image classification domain. 
\end{enumerate}

\section{Related Work}
In this section, we provide a brief overview of existing black-box adversarial attacks and applications of knowledge distillation.  

\subsection{Transferable Adversarial Perturbations}
In this work, we focus on the transferability of an adversarial attack from a white-box model to a black-box one, emulating a black-box attack.  Black-box adversarial attacks can be divided into two categories: query-based and transfer-based. In a query-based attack, an adversary uses an output of the target model to compute an adversarial example. One way to do this is to estimate the gradient of the model to the input object [\cite{bhagoji2018practical,chen2017zoo,ilyasprior,guo2019simple}]. However, these methods usually require a lot of queries to the target model, which makes them infeasible in practice. In a transfer-based attack, an adversary generates adversarial examples by attacking one or several surrogate models [\cite{liu2022delving,qin2023training}]. The transferability of adversarial examples generated for surrogate models to the target model can be improved by utilizing data augmentations [\cite{xie2019improving}], exploiting gradients [\cite{wuskip}], gradient aggregation [\cite{liu2023enhancing}] or direction tuning [\cite{yang2023improving}].

There are plenty of black-box attack methods known, for example, ZOO [\cite{chen2017zoo}] and NES  [\cite{ilyas2018black}].
ZOO attack sequentially adds a small positive or negative perturbation to each pixel of the target image. It then queries the black-box model to estimate the gradient in the vicinity of the target image. %evaluates the resulting images by querying the black box model and makes a gradient estimate based on this information.
NES attack works similarly. However, instead of changing pixel by pixel, a set of random images is generated, which are used to approximately estimate the gradients.

% I am here
Current SOTA methods are Square Attack  [\cite{andriushchenko2020square}], NP-Attack [\cite{bai2020improving}],
% SignHunter [\cite{al2020sign}], 
MCG [\cite{yin2023generalizable}] and Bayesian attack [\cite{li2023making}].
Square Attack works differently. The attack selects an area of the image that is subject to attack and then gradually changes this area as the algorithm runs. And within the selected area, random pixels are selected that are changed.
NP-Attack leverages a neural predictor model to guide the search for adversarial perturbations by predicting the model’s output with fewer queries.
% SignHunter focuses on estimating the sign of the gradient in a query-efficient manner, which is sufficient to craft adversarial perturbations. The key idea is that instead of estimating the full gradient, the attack can estimate just the sign of each gradient component, significantly reducing the number of queries required to perform the attack.
MCG is a meta-learning-based black-box attack that leverages a meta-classifier to generalize adversarial attacks across different black-box models. The idea is to train a meta-classifier to guide the adversarial example generation.
Bayesian attack enhances the transferability of adversarial examples by using a substitute model with Bayesian properties. The key idea is to make the substitute model more Bayesian through techniques like Monte Carlo dropout or stochastic weights, which results in better uncertainty estimation. This improved uncertainty estimation enhances the transferability of adversarial examples crafted on the substitute model to the target black-box model.

Note that Bayesian attack [\cite{li2023making}] belongs to the transfer-based category and implies access to part of the training data of the black-box model. In our work, we assume that an adversary has no access to the training data and, thus, we do not compare our approach against methods from the transfer-based category.

% is present in the comparison

%, where the main focus is on the ASR metric only. Also, this category allows access to the data on which the black-box model is trained, which significantly simplifies the task. Since this work assumes that there is no access to training data, the Bayesian attack and any other methods from the Transfer-Based category are not present in the comparison.

\subsection{Knowledge Distillation and Adversarial Robustness}
Knowledge distillation (KD) is a method to transfer the performance of a large teacher neural network to a smaller, lightweight student neural network [\cite{hinton2015distilling}].  Given a teacher model $T$, the framework is used to train a student network $S$ by solving an optimization problem:

\begin{equation}
    S = \arg\min_{S'} \mathbb{E}_{(x,y) \sim \mathcal{D}} \left[ \alpha \mathcal{L}(S'(x),y) + (1-\alpha)\tau^2 KL(S'(x), T(x)) \right],
\end{equation}
where $\mathcal{D}$ is the distillation dataset, $\mathcal{L}$ is the classification loss function used to assess the performance of the student model, KL is the Kullback-Leibler divergence and $\alpha, \tau$ are the scalar parameters. Knowledge distillation has been used in a large scope of problems, such as model compression [\cite{sun2019patient,wang2019private,li2020few}], data privacy [\cite{lyu2020differentially,chourasia2022knowledge,galichin2024glira,pautov2024probabilistically}], adapted for large language models [\cite{mcdonald2024reducing,gu2024minillm,kang2024knowledge}] and diffusion models [\cite{huang2024knowledge,yao2024progressively,yin2024one}].

It has recently been shown that knowledge distillation can be used to enhance the adversarial robustness of additive perturbations [\cite{papernot2016distillation,kuang2024improving,huang2023boosting}]. In contrast to a large teacher model which can attain a satisfactory level of adversarial robustness, it is challenging to make a small student model both robust and similar to the teacher one in performance [\cite{huang2023boosting}]. To deal with this issue, adversarially robust distillation was proposed [\cite{goldblum2020adversarially}]. This approach takes into account clean predictions [\cite{goldblum2020adversarially}] or probability vectors [\cite{zi2021revisiting}] of robust teacher model during the distillation procedure.  

%In contrast, knowledge distillation can be used to conduct a black-box adversarial attack on the teacher model [\cite{cui2020substitute,zhang2020learning}]. 

\section{Problem Statement}
In this section, we formally discuss a problem statement, introduce the notations used throughout the paper, and formulate the research question. 
\subsection{Adversarial Example for a Classification Neural Network}
Suppose that $f:\mathbb{R}^d \to \Delta^K$ is the classification neural network that maps input object $x \in \mathbb{R}^d$ to the vector $f(x) \in \Delta^K$ of probabilities of $K$ classes and 
\begin{equation}
\label{eq:clf_rule}
    h(f,x) = \arg\max_{i \in [1,\dots,K]} f(x)_i
\end{equation}is the associated classification rule. We begin by formally defining an adversarial example for the given classification neural network and the transferability of an adversarial example between the two networks. 

\begin{definition}[Adversarial Example]
\label{def:adv_ex}
    Suppose that $x \in \mathbb{R}^d$ is the input object correctly assigned to class $y \in [1,\dots, K]$ by the network $f$, namely, $h(f,x) = y.$ Let $\delta>0$ be a fixed constant. Then, the object $x' \in \mathbb{R}^d: \|x-x'\|_2 \le \delta$ is the \emph{untargeted} adversarial example for $f$ at point $x$, if
    \begin{equation}
        \label{eq:adv_example}
        h(f,x') \ne h(f,x). 
    \end{equation}
    If $h(f, x') = t$ for some predefined class index $t$, then $x'$ is called \emph{targeted} adversarial example. 
\end{definition}

\begin{definition}[Transferable Adversarial Example]
    Let $x'$ be the adversarial example computed for the network $f$ at point $x$ and let $g:\mathbb{R}^d\to\Delta^K$  be the separate network. Then, $x'$ is transferable from $f$ to $g$, if
    \begin{equation}
    \label{eq:transfer}
        \begin{cases}
            h(f,x) = h(g,x), \\
            h(f,x') = h(g, x').
        \end{cases}
    \end{equation}
\end{definition}

\subsection{Knowledge Distillation of a Black-Box Model}
In this paper, we focus on using knowledge distillation [\cite{hinton2015distilling}] to construct adversarial perturbations for the given classification model deployed in the black-box setting. Namely, let $T:\mathbb{R}^d \to \Delta^K$ be the black-box teacher model trained on an unknown dataset $\mathcal{D}(T)$ and $S:\mathbb{R}^d \to \Delta^K$ be the white-box student model, possibly of a different architecture, and let $\mathcal{D}(S)$ be its training dataset. To approximate the teacher model, we apply soft-label knowledge distillation, which is done in two steps.  Firstly, the teacher model is used to collect the training dataset for the student model. In our setting, we use a hold out dataset $\mathcal{D}_h = \{(x_i,y_i)\}_{i=1}^m$ to construct $\mathcal{D}(S):$ 
\begin{equation}
\label{eq:student_data}
    \mathcal{D}(S) = \{(x_i, T(x_i))\}_{i=1}^m,
\end{equation}
where $x_i \in \mathcal{D}_h$ and  $T(x_i) \in \Delta^K$.
Then, the student network $S$ is trained on the dataset $\mathcal{D}(S)$ by minimizing an empirical risk
\begin{equation}
    L(S, \mathcal{D}(S)) = \frac{1}{m}\sum_{(x_i, y_i) \in \mathcal{D}(S)} l(S,x_i, y_i),
\end{equation}
where $l(S,x,y) = -\log (S(x)_y)$ is the cross-entropy loss function.

When the student model is trained, we ask the following research question. Given  $x \in \mathbb{R}^d: h(S,x) = h(T,x)$ and $\delta > 0$ from the definition \ref{def:adv_ex}, is it possible to compute an adversarial example for the model $S$ at point $x$ which is \emph{provably} transferable to $T$? In the next section, we answer this question and propose a knowledge distillation-based adversarial attack with transferability guarantees.

%-----------------
%-----------------
%-----------------
%-----------------
%-----------------

\section{Methodology}
In this section, we describe the proposed approach to generate adversarial examples for the black-box teacher model via knowledge distillation. In the last subsection, we prove that, under several assumptions,  our approach generates an adversarial example that is transferable to the teacher model within the finite number of iterations. 
\subsection{Model Mimic Attack: Student Follows its Teacher}

To perform an adversarial attack on the black-box teacher model $T$, we first apply soft-label knowledge distillation and obtain the white-box student model $S$. The training dataset for the student model is constructed by querying the teacher model and collecting its predictions for the points from the hold-out dataset $\mathcal{D}_h$, possibly disjoint from the teacher's training dataset ($\mathcal{D}(T):$  $\mathcal{D}_h \cap \mathcal{D}(T) = \emptyset$). In our setup, we use the test subset of the teacher's dataset as the hold-out dataset $\mathcal{D}_h.$

Recall that $\mathcal{D}(S) = \{(x_i, T(x_i))\}_{i=1}^m, $ according to \eqref{eq:student_data}.  Assuming that the student model has enough learning capability, we train it until it perfectly matches the teacher model on $\mathcal{D}(S)$, namely,
\begin{equation}
\label{eq:perfect_student}
\begin{cases}

    h(S,x_i) = h(T,x_i) = y_i  \\
    \|S(x_i) - T(x_i)\|_\infty < \frac{\varepsilon}{4},

\end{cases}
\end{equation}
$\text{for all} \ (x_i, y_i)  \in \mathcal{D}(S)$, where $\varepsilon>0$ is the predefined constant. In \eqref{eq:perfect_student}, the second condition reflects the ability of the student model to confidently mimic the teacher model on $\mathcal{D}(S).$

\subsection{Model Mimic Attack: Student Under Attack}
In this subsection, we describe a procedure to generate a single adversarial example for the student model.

When the student model is trained, we perform the white-box adversarial attack on it.  To do so, we use Projected Gradient Descent [PGD, \cite{madry2018towards}].  Given input object $x \in \mathbb{R}^d$ of class $y \in [1,\dots, K]$ correctly predicted by both teacher and student models,  PGD performs iterative gradient ascent to find an adversarial example $x'$ within $U_\delta(x)$, the $\delta-$neighborhood of $x.$ Namely, for all $t \in [1,\dots,M],$
\begin{equation}
\label{eq:pgd}
    \begin{cases}
        x^{t+1} = \text{Proj}_{U_{\delta}(x)} \left[x^{t} + \alpha \sign \nabla_{x^t}L(S, x^t, y)\right], \\
        x^1 = x, x' = x^{M},
    \end{cases}
\end{equation}
where $\alpha>0$ is the value of a single optimization step, $M$ is the maximum number of PGD iterations, $\text{Proj}_{U_\delta(x)}$ is the projection onto $U_\delta(x)$, defined as
\begin{equation}
    U_\delta(x) = \{x': \|x-x'\|_2 \le \delta\},
\end{equation}
and $L(S, x^t, y)$ is the loss function reflecting the error of the model $S$ on the sample $(x^t,y)$. In our setting, $L(S, x^t, y)$ is the cross-entropy loss.

When the adversarial example $x'$ for the student model $S$ is found, we verify if it transfers to the teacher model, namely, if $h(S,x') = h(T,x').$ Not that  $x'$ does not have to be a transferable adversarial example. If $h(S,x') \ne h(T,x'),$ then we add $x'$ to the training dataset $\mathcal{D}(S)$ of the student model and repeat both the training of $S$ and adversarial attack on it. 

\begin{remark}
To increase the computational efficiency of the attack, we generate not a single adversarial example $x'$ for the student model, but a batch $\{x'_1, \dots, x'_l\}$ of $l$ adversarial examples. The pseudo-code of the proposed method is presented in the Algorithm \ref{alg:mma}. Note that we use a Projected Gradient Descent attack because of its simplicity; our approach is not limited to a specific type of white-box attack. 
\end{remark}

\begin{algorithm}
    \caption{Model Mimic Attack}\label{alg:mma}
    \begin{algorithmic}[1]
    \Require Black-box teacher model $T$,  input object $x$ of class $y$, distance threshold $\delta$, gradient step $\alpha$, maximum number of PGD iterations $M$, maximum number of distillation iterations $N$, hold-out dataset $\mathcal{D}_h$, the number $l$ of adversarial examples to generate for the student model $S_i$
    \Ensure Set of student models $S_1, \dots, S_N$, the set  $AE(T)$ of adversarial examples for the teacher model $T$ 
    %\State Take random $K$ images with labels 
    \State{$z \gets (x, T(x))$}
    \Comment{compute the logits of $T$ at the target point}
    \State {$\mathcal{D}(S) \gets \{(x_i, T(x_i))\}_{i=1}^m$ }
    \Comment{compute the training set $\mathcal{D}(S)$ according to the \eqref{eq:student_data}}
    \State $\mathcal{D}(S_1) \gets \mathcal{D}(S) \cup z$ 
    \Comment{initialize the training set for the first student model $S_1$}
    \State $AE(T) \gets \emptyset$ 
    \Comment{initialize the set of adversarial examples for the teacher model $T$}
    \For{$i = 1$ to $N$} 
        \State {$S_i \gets \texttt{train}(\mathcal{D}(S_i))$}
        \Comment{train the student model $S_i$  using $\mathcal{D}(S_i)$}
        % \State {$AE(S_i) \gets \emptyset$}
        % \Comment{initialize the set of adversarial examples for the student model $S_i$ }
        %\State $white\_box\_attack\_set\_for\_train \gets \O$
        \For{$j = 1$ to $l$} 
            %\State Add random noise $\delta$ to target image
            %\State Try to find white-box attack example $WB_i$
            \State $(x'_j, y'_j) \gets PGD(\alpha,\delta,S_i,(x,y))$ 
            \Comment{compute an adversarial example for the student model $S_i$ according to \eqref{eq:pgd}}
            \If{$h(S_i, x'_j) = h(T, x'_j)$}
            \Comment{check if the adversarial example transfers from $S_i$ to $T$}
                \State 
                $AE(T) \gets AE(T) \cup \{(x'_j, y'_j)\}$
                \Comment{update the set of adversarial examples for the model $T$}
            \EndIf
            %\Else 
            \State {$\mathcal{D}(S_{i+1}) \gets \mathcal{D}(S_{i}) \cup \{(x'_j, T(x'_j))\}$}
            \Comment{update the training set for the model $S_{i+1}$}
            % \EndIf 
            %\State $white\_box\_attack\_set\_for\_train \gets white\_box\_attack\_set\_for\_train \cap WB_i$
        \EndFor
    % \State Check transferable $white\_box\_attack\_set$ attack examples on black-box model $T$
    % % \If{transferable more than $M/2$ examples}
    % %     \State \textbf{break}
    % % \Else
    % \State Calculate black-box model logits $T(white\_box\_attack\_set\_for\_train) = train\_set\_adv$
    % \State $train\_set \gets train\_set \cap train\_set\_adv$
    % \State Reset white-box model $S$ weights
    % \State Train simple white-box model $S$ use $train\_set$
    % % \EndIf
\EndFor
\end{algorithmic}
\end{algorithm}

\subsection{Model Mimic Attack: Provably Transferable Adversarial Examples}
It should be mentioned that, under several assumptions, the Algorithm \ref{alg:mma} is \emph{guaranteed} to find an adversarial example that is transferable from the student model to the teacher model within the finite number of iterations. Namely, let $T$ be the teacher model and $S_i$ be the student model on $i'$th iteration with the corresponding training dataset $\mathcal{D}(S_i).$ Let $x \in \mathbb{R}^d$ be the input object correctly assigned by the teacher model to class $y \in [1,\dots, K],$ and $\delta>0$ be the distance threshold. Suppose that for every $i \in \mathbb{Z}_+, $ the learning capability conditions from the \eqref{eq:perfect_student} hold. Then, the following theorem holds.

\begin{theorem}
\label{th:main}
    If $f_i = S_i - T$ be the functions with the bounded gradient in $U_\delta(x)$ for every $i \in \mathbb{Z}_{+}$ and let  
    \begin{equation}
        \beta = \sup_{f_i}\sup_{x' \in U_\delta(x)}\|\nabla f_i(x')\|_F.
    \end{equation}
    
    Suppose that for every $i \in \mathbb{Z}_{+}$,  Algorithm \ref{alg:mma} yields an adversarial example for the model $S_i$ within the $\delta-$neighborhood of $x$. Then, exists $N \in \mathbb{Z}_{+}$ such that Algorithm \ref{alg:mma} on $N'$th iteration yields an adversarial example transferable from $S_N$ to $T$.
\end{theorem}
\begin{proof}
    Let $\{x'_i\}_{i=1}^\infty$ be the sequence of adversarial examples generated by Algorithm \ref{alg:mma} such that $\|x'_i - x\|_2 \le \delta$ and $x'_i$ is the adversarial example for the model $S_i$. Then, the sequence $\{x'_i\}_{i=1}^\infty$ is bounded in $U_\delta(x)$ and, hence, there exists the subsequence $\{x'_{i_j}\}_{j=1}^\infty$ such that exists
    \begin{equation}
        \lim_{j \to \infty } x'_{i_j} = z \in U_\delta(x).
    \end{equation}
    Without the loss of generality, assume that $z \ne x$ and let $\{x'_{i_j}\}_{j=1}^\infty = \{z_i\}_{i=1}^\infty.$ 

    Then, 
    \begin{align}
    \label{eq:main_ineq}
       & |\|f_{i+1}(x)\|_\infty - \|f_{i+1}(z_{i+1})\|_\infty| \le   \|f_{i+1}(x) - f_{i+1}(z_{i+1})\|_\infty \le \\ 
       & \le \|f_{i+1}(x) - f_{i+1}(z_i)\|_\infty +     \|f_{i+1}(z_i) - f_{i+1}(z_{i+1})\|_\infty \le \nonumber \\
       & \le \|f_{i+1}(x)\|_\infty + \|f_{i+1}(z_i)\|_\infty +     \|f_{i+1}(z_i) - f_{i+1}(z_{i+1})\|_\infty \le \\
       & \le \frac{\varepsilon}{4} + \frac{\varepsilon}{4} + \|f_{i+1}(z_i) - f_{i+1}(z_{i+1})\|_\infty, \nonumber
    \end{align}
    where the last inequality is due to conditions from \eqref{eq:perfect_student}.

    According to the mean value theorem, 
    \begin{equation}
        f_{i+1}(z_i) - f_{i+1}(z_{i+1}) = \nabla f_{i+1}(\tau_{i+1})^\top (z_i - z_{i+1}), 
    \end{equation}
    for some $\tau_{i+1} \in [z_i, z_{i+1}] \subset U_\delta(x)$.

    Since $\lim_{i \to \infty} z_i = z$, then $\lim_{i \to \infty} \|z_{i} - z_{i+1}\|_F = 0$ and $\exists N\in \mathbb{Z}_{+}: \|z_{N-1} - z_{N}\|_F < \frac{\varepsilon}{4\beta}.$

    Then,
    \begin{align}
    \label{eq:mean_val}
      \| f_{N}(z_{N-1}) - f_{N}(z_{N})\|_\infty & \le \| f_{N}(z_{N-1}) - f_{N}(z_{N})\|_F   \le \|\nabla f_{N}(\tau_N)\|_F \|z_{N-1} - z_{N}\|_F < \\& <  \frac{\varepsilon}{4} \nonumber .  
    \end{align}
    Substituting \eqref{eq:mean_val} into \eqref{eq:main_ineq}, we get

    \begin{equation}
        |\|f_{N}(x)\|_\infty - \|f_{N}(z_{N})\|_\infty| < \frac{3\varepsilon}{4}, \ \text{yielding}\   \|f_{N}(z_{N})\|_\infty < \|f_{N}(x)\|_\infty + \frac{3\varepsilon}{4} = \varepsilon.
    \end{equation}
    By setting $\varepsilon$ to be small enough, for example, 
    \begin{equation}
         \varepsilon < \frac{p_1 - p_2}{2}, \ \text{where $p_1, p_2$  are the two largest components of $S_N(z_N)$,} 
    \end{equation}
    we get $h(S_N, z_N) = h(T, z_N)$, what finalizes the proof. 
   
\end{proof}

\section{Experiments}

This section will describe the experiments and everything needed to reproduce them. In particular, a description of the datasets, a method for evaluating the experiments, a description of the methods we compare with, and the methodology for conducting the experiments.

\subsection{Setup of Experiments}

\paragraph{Datasets and Training.}

In our experiments, we use CIFAR-10 and CIFAR-100 [\cite{krizhevsky2009learning}] as the training datasets for the teacher model. We use ResNet50 [\cite{he2016deep}] as the teacher model $T$, which was trained for $250$ epochs to achieve high classification accuracy (namely, 82\% for CIFAR-10 and 47\% for CIFAR-100. To train the teacher model, we use the SGD optimizer with the learning rate of $0.1$, the weight decay of $10^{-4}$, and the momentum of $0.9$.

\paragraph{MMAttack Setup.}

We use ResNet18 and SmallCNN as the white-box student models. The architecture of SmallCNN is presented in the Appendix. 
We conduct the PGD attack on the student models with the following parameters: the number of PGD steps is set to be $M=10$, the gradient step is set to be $\alpha = 0.005$, the distance threshold is set to be $\delta = 0.05$. The detailed architecture of the Small CNN model is presented in the appendix \ref{Appendix:A}.

%: 3 Conv2d layers with parameters in\_channels = 3, 64 and 128 respectively, out\_channels = 64, 128, 256 respectively, kernel\_size = 3, stride = 1, padding = 1. Then the three-dimensional tensor is converted to a 1D tensor. In the end, two fully connected layers with parameters in\_features = 256 * 4 * 4, 512 respectively, out\_features = 512, the number of classes respectively. Between all layers there is a ReLu activation layer, after the convolution layers there are also MaxPool2d pooling layers with parameters kernel\_size = 3, stride = 2

% epsilon\_min\_bound = 0.01, epsilon\_max\_bound = 0.1, 
% alpha\_min\_bound = 0.01, alpha\_max\_bound = 0.01 
% was used as the attack on the white-box model. During the attack, the interval between the minimum and maximum boundaries of epsilon 
% and alpha 
% was divided by the number of attacks that needed to be generated and compared pairwise. 

%The initial set for the pair of ResNet50 and ResNet18 models (black and white boxes, respectively) was 300, for the pair ResNet50, SmallCNN --- 10. The number of sought attacks for the pair ResNet50, ResNet18 --- 30, for the pair ResNet50, SmallCNN --- 10.

\paragraph{Methods for Comparison.}

In this section, we briefly list the set of methods we compare our approach against. We evaluate MMAttack against ZOO [\cite{chen2017zoo}], NES [\cite{ilyas2018black}] as the main competitors.  Among the black-box attack methods based on a random search, we choose Square attack [\cite{andriushchenko2020square}] as the state-of-the-art in terms of an average number of queries to conduct an attack. In the group of methods using gradient estimation, NP-Attack [\cite{bai2020improving}] is among the most efficient attacks. In the category of combined methods, we choose MCG [\cite{yin2023generalizable}]. The hyperparameters that were used in the experiments with Methods for Comparison are described in detail in the appendix \ref{Appendix:B}.
%In the category of transfer-based attacks, we choose Bayesian attack [\cite{li2023making}].

Note that the MCG  algorithm originally assumes the training on the data from a distribution that is close to the teaches model's one, which in general may not be known. Here, we highlight that our method does not have such a limitation. % that were not directly in the training set of the black-box model. However, this is a limitation of these methods. %And the algorithm proposed in this paper does not have such a limitation

\paragraph{Evaluation Protocol.}

%The median query number (MQN) is also sometimes used but is much less popular. Similarly, MQN captures the middle value among all query numbers of successfully crafted adversarial examples.

To illustrate the efficiency of the proposed approach, we report the Average Query Number (AQN) and demonstrate the trade-off between AQN and the Average Success Rate (ASR).  AQN denotes the number of queries required to generate all the adversarial examples for the black-box model, averaged over all the examples.  ASR measures the fraction of adversarial examples assigned to a different class in an untargeted attack setting or to the predefined other class in the targeted attack setting. For AQN, a lower value indicates better attack performance, while for ASR, a higher value indicates a better attack performance. Note that both metrics are calculated over successful adversarial attacks only. In this paper, the emphasis is made on minimizing the AQN.  

% In our work, the main emphasis is given to the value of the AQN metric: we want to compare adversarial attacks by the number of queries required to find the first adversarial example for the black-box model. Later in the paper, we illustrate the trade-off between AQN and ASR for the proposed approach. 

% However, the study of the ASR metric of the proposed method will also be given in the Influence of hyperparameters subsection, where it will be shown how the hyperparameters of our algorithm affect both metrics. To evaluate methods using the AQN metric, the methods evaluated the portability of the attack at each iteration of the main loop. This was done for each method under consideration. This modification allows us to evaluate and record the minimum number of queries required by the methods to find the first portable attack.

\subsection{Results of Experiments}

\begin{table}[t]
    \caption{Comparison of black-box attack methods. We report the average number of queries (AQN) required to generate the first adversarial example for the black-box model. Here, $\delta$ denotes the value of the maximum possible distance from the target point in terms of $l_\infty$ norm. ($\textbf{Lit}$) denotes the metric values taken from the literature.} 
    \label{tab:table_queries_first}
\begin{center}

\begin{tabular}{cccc}
\toprule
 $\mathcal{D}$ & Attack & $\delta$ & AQN (\textdownarrow) \\
\midrule
% \hline \hline
 \multirow{7}{*}{CIFAR-10} 
 &  ZOO [\cite{chen2017zoo}] & $0.05$ & $ \ge 3 \times 10^5 $ \\
 &  NES [\cite{ilyas2018black}] & $0.1$ & $ 3578 $ \\
 &  Square [\cite{andriushchenko2020square}] & $0.1$ & $ 368 $ \\
 &  NP-Attack [\cite{bai2020improving}] ($\textbf{Lit}$) & $0.05$ & $500 $ \\
 % &  SignHunter [\cite{al2020sign}] ($\textbf{Lit}$) &  & $ \textcolor{red}{600} $ \\
 &  MCG [\cite{yin2023generalizable}] ($\textbf{Lit}$) & $0.1$ & $ 130 $ \\
 % &  Bayesian [\cite{li2023making}] ($\textbf{Lit}$) &  & $ \textcolor{red}{?} $ \\
 &  MMAttack resnet18 ($\textbf{ours}$) & $0.05$ & $ 530 $ \\
 &  MMAttack SmallCNN ($\textbf{ours}$) & $0.05$ & $ \textbf{32.8} $ \\
 \midrule 
\multirow{7}{*}{CIFAR-100} 
 &  ZOO [\cite{chen2017zoo}] & $0.05$  & $ \ge 3 \times 10^5 $ \\
 &  NES [\cite{ilyas2018black}] & $0.1$  & $ 4884 $ \\
 &  Square [\cite{andriushchenko2020square}] & $0.1$  & $ 193 $ \\
 &  NP-Attack [\cite{bai2020improving}] & $0.05$  & $ 325 $ \\
 % &  SignHunter [\cite{al2020sign}] & $  $ \\
 &  MCG [\cite{yin2023generalizable}] ($\textbf{Lit}$) & $0.1$  & $ 48 $ \\
 % &  Bayesian [\cite{li2023making}] & $ $ \\
 &  MMAttack resnet18 ($\textbf{ours}$) & $0.05$  & $ 407 $ \\
 &  MMAttack SmallCNN ($\textbf{ours}$) & $0.05$  & $ \textbf{24} $ \\

\bottomrule
\end{tabular}
\end{center}
\end{table}

In the experiments, ZOO, NES, and Square attack methods were executed $100$ times with different random seeds, NP-Attack, 
% SignHunter, 
MCG, 
% Bayesian, 
MMA methods were executed $30$ times. %As was written earlier, if the experimental results for the methods described in the literature did not match the results obtained in our implementation, then the results with the smallest number of average queries were preferred.

Table \ref{tab:table_queries_first} shows a comparison of existing SOTA methods and the MMAttack method proposed in this work with two different substitute model architectures on the CIFAR-10 and CIFAR-100 datasets. The best results are highlighted in bold. It can be seen that the MMAttack method with the substitute model SmallCNN outperforms the competitors in terms of the AQN metric. (Table data for the MCG method on the CIFAR-100 dataset was taken from  [\cite{yin2023generalizable}]. Table data for the MCG and NP-Attack methods for the CIFAR-10 dataset were taken from  [\cite{zheng2023blackboxbench}]).

Note that if the results of a method presented in the literature do not match the results obtained in our implementation, then the result with the smallest number of average queries is reported. In the tables, the results taken from the literature are marked as ($\textbf{Lit}$).

% The results were taken either from the articles of the authors of the methods or the article "BlackboxBench: A Comprehensive Benchmark of Black-box Adversarial Attacks" \cite{zheng2023blackboxbench}. 

\subsection{Ablation Study}

Note that the success of our black-box attack crucially depends on the architecture of the white-box student model. On the one hand, the student model does not have to have many training parameters since it implies several retraining iterations. On the other hand, it has to have enough learning capacity to mimic the behavior of the black-box model in the vicinity of the target point. In Table \ref{tab:black-white-comp}, we report the AQN values for the different pairs of teacher and student models on the CIFAR-10 dataset. Together with the average number of queries, we report the size of the initial training dataset $\mathcal{D}(S_1)$ of the student model and the number of adversarial examples to generate for the student model, $l$.  We found that the simpler the architecture of the student model, the fewer queries to the teacher model are required to conduct a successful attack. 

%I AM HERE

%The main hyperparameter is the architecture of the white-box model, which is used to model the behavior of the black-box in the vicinity of the target point.

% Table \ref{tab:black-white-comp} compares the AQN value for different pairs of black and white-box models on dataset CIFAR-10. 

% It can be seen that when a more complex model is used as a white-box model, the required number of queries to the black-box model increases dramatically. At the same time, when using the simplest model as a white-box model, the best results are obtained. 

% This is because the white-box model must model the behavior of the black-box model in a small neighborhood around the target point as accurately as possible. 

% It follows that the choice of the architecture of the white-box model must satisfy two criteria: it must be as simple as possible to learn quickly, but at the same time have sufficient generalization ability to effectively learn on 300 points, which are determined by the hyperparameters: initial set size (random points for the first train white-box), adv examples batch size (how many adv examples white-box attack generate), C.

\begin{table}[t]
    \caption{Impact of hyperparameters on the performance of the MMAttack.}
    \label{tab:black-white-comp}
\begin{center}

\begin{tabular}{ccccc}
\toprule
 Teacher model, $T$ & Student model, $S$ & Initial dataset size, $|\mathcal{D}(S_1)|$ & $l$ & AQN (\textdownarrow) \\
\midrule
 ResNet101 & ResNet34 & 800 & 400 & $ 4520 $ \\
 ResNet50 & ResNet34 & 600 & 400 & $ 4160 $ \\
 ResNet101 & ResNet18 & 600 & 300 & $ 1560 $ \\
 ResNet50 & ResNet18 & 600 & 200 & $ 530 $ \\
 ResNet34 & ResNet18 & 300 & 30 & $ 455 $ \\
 ResNet101 & SmallCNN & 10 & 10 & $ 37.7 $ \\
 ResNet50 & SmallCNN & 10 & 10 & $ 32.8 $ \\
 ResNet34 & SmallCNN & 10 & 10 & $ 34 $ \\
 ResNet50 & SmallCNN & 5 & 5 & $ - $ \\

\bottomrule
\end{tabular}
    
\end{center}
\end{table}

The initial set size, $|\mathcal{D}(S_1)|$, represents the number of random data points to be included in the initial training dataset of the white-box student model. It can be seen from Table \ref{tab:black-white-comp}, that the more complex the student model is, the larger this parameter should be. The same is true for the number of adversarial examples for the student model, $l$. 

Note that there is no AQN value corresponding to  $|\mathcal{D}(S_1)|=5$ and $l=5$. This is because the Algorithm \ref{alg:mma} does not succeed in finding a single adversarial example for the black-box teacher model until it reaches the maximum iterations threshold. % attack method did not find an attack at a certain iteration and stopped.

%From the general recommendations, it can be noted that the parameters (initial set size) and (adv examples batch size) can be reduced as long as the condition is met that at least one attack is found for the white-box model after training. Therefore, reducing the parameter (adv examples batch size) is more critical, as the risk of not finding an attack increases. However, completely zeroing the parameter (initial set size) is also bad, since then the model will have no idea about the behavior of the model on other classes, which may not allow finding an attack within a limited disturbance.

It is also worth mentioning that Model Mimic Attack implies a certain trade-off between ASR and AQN metrics. At the start, when the size of the training dataset of the student model is relatively small and very few iterations of knowledge distillation are passed, the algorithm is less likely to find an adversarial example for the teacher model. In contrast, after more distillation iterations, the algorithm tends to find more transferable adversarial examples on each iteration. 
% The described balance can be seen in tables \ref{tab:ASR-AQN_balance_10} and \ref{tab:ASR-AQN_balance_100}. 
In tables \ref{tab:ASR-AQN_balance_10} and \ref{tab:ASR-AQN_balance_100}, we show the trade-off between the ASR and AQN metrics from one distillation iteration to another: when the number of passed distillation iterations increases, so does the number of queries to the teacher model used to collect additional training samples for the student model by that iteration, $QN_1$. In contrast, the number of queries remaining to find an attack on the black-box model, $QN_2$, decreases (here, we fix the total number of queries to be $QN_1 + QN_2 = 200$).  

% is affected by the number of queries that will be given as training queries from the total budget of 200 queries. And also how an increase in the initial training budget affects the AQN metric and the number of attacks that will be generated by MMAttack.

However, if the goal is not to obtain the minimum value of the AQN metric, but to improve the trade-off between the ASR and AQN metrics, one could run several cycles of the algorithm to better study the behavior of the teacher model in the vicinity of the target point.

\begin{table}[ht]
\caption{Trade-off between ASR and AQN metrics for MMAttack, CIFAR-10 dataset. $QN_1$ represents the number of data points  added to the training dataset of the student model by corresponding iteration; $QN_2$ represents the attack budget, or upper bound of the number of queries to find an attack on the black-box model.}
    \label{tab:ASR-AQN_balance_10}
\begin{center}
    
\begin{tabular}{cccccc}
\toprule
% \hline
Iteration number & $QN_1$ & $QN_2$ & Number of generated attacks & ASR (\textuparrow) & AQN (\textdownarrow) \\ \hline
1 & 10 & 190 & 121.67 & 0.66 & 2.36 \\  
2 & 20 & 180 & 112.05 & 0.68 & 2.50 \\  
3 & 30 & 170  & 105.57 & 0.67 & 2.67 \\  
4 & 40 & 160  & 99.38  & 0.67 & 2.86 \\  
5 & 50 & 150 & 93.10  & 0.67 & 3.03 \\  
6 & 60 & 140 & 87.48  & 0.67 & 3.24 \\  
7 & 70 & 130 & 81.00  & 0.67 & 3.49 \\  
8 & 80 & 120 & 75.14  & 0.68 & 3.74 \\  
9 & 90 & 110 & 69.19  & 0.68 & 4.05 \\  
10 & 100 & 100 & 62.81  & 0.68 & 4.43 \\  
11 & 110 & 90 & 56.43  & 0.70 & 4.83 \\  
12 & 120 & 80 & 50.43  & 0.69 & 5.50 \\  
13 & 130 & 70 & 43.62  & 0.69 & 6.35 \\  
14 & 140 & 60 & 37.43  & 0.68 & 7.42 \\  
15 & 150 & 50 & 30.95  & 0.67 & 9.13 \\  
16 & 160 & 40 & 25.00  & 0.68 & 11.16 \\  
17 & 170 & 30 & 19.05  & 0.68 & 14.62 \\  
18 & 180 & 20 & 12.62  & 0.74 & 20.39 \\  
19 & 190 & 10 & 6.86   & 0.72 & 38.38 \\  \hline
\end{tabular}
\end{center}
\end{table}

\begin{table}[ht]
\caption{Trade-off between ASR and AQN metrics for MMAttack, CIFAR-100 dataset. $QN_1$ represents the number of data points  added to the training dataset of the student model by corresponding iteration; $QN_2$ represents the attack budget, or upper bound of the number of queries to find an attack on the black-box model.} 
    \label{tab:ASR-AQN_balance_100}
\begin{center}

\begin{tabular}{cccccc}
\toprule
% \hline
Iteration number & $QN_1$ & $QN_2$ & Number of generated attacks & ASR (\textuparrow) & AQN (\textdownarrow) \\ \hline
1 & 10 & 190  & 163.17 & 0.84 & 1.38 \\  
2 & 20 & 180  & 153.17 & 0.85 & 1.46 \\  
3 & 30 & 170  & 144.90 & 0.85 & 1.54 \\  
4 & 40 & 160  & 135.86 & 0.85 & 1.64 \\  
5 & 50 & 150  & 127.34 & 0.85 & 1.75 \\  
6 & 60 & 140  & 119.14 & 0.85 & 1.87 \\  
7 & 70 & 130  & 110.48 & 0.85 & 2.02 \\  
8 & 80 & 120  & 102.03 & 0.85 & 2.19 \\  
9 & 90 & 110  & 93.59  & 0.85 & 2.39 \\  
10 & 100 & 100  & 85.52  & 0.85 & 2.63 \\  
11 & 110 & 90  & 76.97  & 0.84 & 2.92 \\  
12 & 120 & 80  & 68.48  & 0.85 & 3.25 \\  
13 & 130 & 70  & 59.97  & 0.86 & 3.69 \\  
14 & 140 & 60  & 51.62  & 0.86 & 4.30 \\  
15 & 150 & 50  & 42.97  & 0.86 & 5.14 \\  
16 & 160 & 40  & 34.72  & 0.86 & 6.37 \\  
17 & 170 & 30  & 26.41  & 0.85 & 8.42 \\  
18 & 180 & 20  & 17.66  & 0.88 & 12.28 \\  
19 & 190 & 10  & 8.72   & 0.90 & 24.08 \\ \hline
\end{tabular}
\end{center}
\end{table}

% In the first iterations of the algorithm, the probability of attack transfer is inferior to the methods described in the literature. However, if the focus is not on minimizing the AQN metric, then the most effective strategy would be to run several cycles of the algorithm to better study the behavior in the vicinity of the target point. 

The choice of the white-box attack method plays an important role in finding the transferable adversarial example: on one hand, the more powerful the white-box attack is, the more frequently an adversarial example will be found for the student model; on the other hand, the faster the attack is, the more distillation iterations can be performed within a limited time. In this work, a projected gradient descent (PGD) attack with the $l_{\infty}$ norm constraint is used, but the method is not limited to any specific type of white-box attack. It is possible to use variants of the white-box attack with $l_2$ or $l_1$ constraints, to conduct an attack in a targeted setting or use more complicated attack methods. In any case, MMAttack is expected to have similar properties.
The optimal choice depends on the specific domain and the effectiveness of each white-box attack method on a given dataset.

% \subsection{Discussions and Future work}

% The following vectors can be highlighted as directions for future work:

% \textcolor{red}{and?}

% \begin{enumerate}
%     \item 
    
%     % It is worth mentioning that the guarantees on the transferability of adversarial examples are given for the soft-label distillation. Note that if the student models do not mimic the behavior of the teacher model well enough, the attack on the latter is not guaranteed to be found. 
    
%     % Study of theoretical guarantees for the case of a hard-label black-box model
%     \item Combination of the ideas of the proposed work and the MCG algorithm. Combining neural network modeling using surrogate models and training a set of replacement models to build an effective ensemble of models for obtaining attack images
%     \item As well as a study of the proposed approach in other domains. In particular, for attacking LLM models 
% \end{enumerate}

\section{Limitations}
Note that the transferability guarantee from Theorem \ref{th:main} is given for the soft-label distillation. It is worth mentioning that the Theorem can not be adapted to the hard-label distillation without significant changes. Instead, to provide the transferability guarantee in hard-label distillation, when the teacher model outputs the predicted class label only, one can estimate the \emph{probability of transferability} of an adversarial example within the finite number of iterations, conditioned on the white-box attack. If the lower bound of this probability is separated from zero, one can estimate the expected number of distillation iterations required to yield the transferable adversarial example.

%same guarantee can be provided for the setting when the black-box model outputs the logit of the predicted class only.%, if several additional assumptions are made: firstly, the functions $f_i$ should have the small deviations of the fixed $c'$th component, the $c$ 

\section{Conclusion and Future Work}
In this paper, we propose the Model Mimic Attack, the first framework to compute adversarial perturbations for a black-box neural network that is guaranteed to find an adversarial example for the latter. To conduct an attack, we apply knowledge distillation to obtain the student model, which is essentially the functional copy of the black-box teacher network. Then, we perform the white-box adversarial attack on the student model and theoretically show that, under several assumptions, the attack transfers to the teacher model. We demonstrate experimentally that a successful adversarial attack can be found within a small number of queries to the target model, making the approach feasible for practical applications. Possible directions for future work include an extension of the transferability guarantees to the hard-label distillation and adaptation of the proposed method for other domains, in particular, for attacking large language models.  

%\subsubsection*{Acknowledgments}

\bibliography{iclr2025_conference}
\bibliographystyle{iclr2025_conference}

\appendix
\section{Appendix: Architecture of SmallCNN}
\label{Appendix:A}

\begin{verbatim}
SmallCNN(
  (features): Sequential(
    (0): Conv2d(3, 64, kernel_size=(3, 3), stride=(1, 1), 
         padding=(1, 1))
    (1): ReLU(inplace)
    (2): MaxPool2d(kernel_size=2, stride=2, padding=0, 
         dilation=1, ceil_mode=False)
    (3): Conv2d(64, 128, kernel_size=(3, 3), stride=(1, 1), 
         padding=(1, 1))
    (4): ReLU(inplace)
    (5): MaxPool2d(kernel_size=2, stride=2, padding=0, 
         dilation=1, ceil_mode=False)
    (6): Conv2d(128, 256, kernel_size=(3, 3), stride=(1, 1), 
         padding=(1, 1))
    (7): ReLU(inplace)
    (8): MaxPool2d(kernel_size=2, stride=2, padding=0, 
         dilation=1, ceil_mode=False)
  )
  (classifier): Sequential(
    (0): Linear(in_features=4096, out_features=512, bias=True)
    (1): ReLU(inplace)
    (2): Linear(in_features=512, out_features=10 or 100, 
         bias=True)
  )
)
\end{verbatim}

% \begin{tikzpicture}[
%     conv/.style={draw, fill=blue!30, minimum width=2cm, minimum height=1cm},
%     relu/.style={draw, fill=green!30, minimum width=2cm, minimum height=1cm},
%     pool/.style={draw, fill=yellow!30, minimum width=2cm, minimum height=1cm},
%     fc/.style={draw, fill=red!30, minimum width=2cm, minimum height=1cm},
%     node distance=1.5cm
% ]

% % Input
% \node (input) [minimum width=2cm, minimum height=1cm] {Input (batch\_sizex3x32x32)};

% % Conv layers
% \node (conv1) [conv, below of=input] {Conv2D (3, 64, 3x3)};
% \node (relu1) [relu, below of=conv1] {ReLU};
% \node (pool1) [pool, below of=relu1] {MaxPool (2x2)};

% \node (conv2) [conv, below of=pool1] {Conv2D (64, 128, 3x3)};
% \node (relu2) [relu, below of=conv2] {ReLU};
% \node (pool2) [pool, below of=relu2] {MaxPool (2x2)};

% \node (conv3) [conv, below of=pool2] {Conv2D (128, 256, 3x3)};
% \node (relu3) [relu, below of=conv3] {ReLU};
% \node (pool3) [pool, below of=relu3] {MaxPool (2x2)};

% % Fully connected layers
% \node (fc1) [fc, below of=pool3] {FC (4096, 512)};
% \node (relu4) [relu, below of=fc1] {ReLU};
% \node (fc2) [fc, below of=relu4] {FC (512, 10/100)};

% % Draw arrows
% \draw[->] (input) -- (conv1);
% \draw[->] (conv1) -- (relu1);
% \draw[->] (relu1) -- (pool1);
% \draw[->] (pool1) -- (conv2);
% \draw[->] (conv2) -- (relu2);
% \draw[->] (relu2) -- (pool2);
% \draw[->] (pool2) -- (conv3);
% \draw[->] (conv3) -- (relu3);
% \draw[->] (relu3) -- (pool3);
% \draw[->] (pool3) -- (fc1);
% \draw[->] (fc1) -- (relu4);
% \draw[->] (relu4) -- (fc2);

% \end{tikzpicture}

\section{Appendix: Hyperparameters of the compared attack methods}
\label{Appendix:B}

\begin{table}[ht]
    \centering
    \caption{Hyperparameters of the compared attack methods}
    \begin{tabular}{@{}ll@{}}
        \toprule
        \textbf{Method} & \textbf{Hyperparameters} \\ \midrule
        ZOO attack & 
        \begin{tabular}[c]{@{}l@{}}
            $\epsilon = 0.05$ \\
            num\_iterations = 5000 \\
            learning\_rate = 0.01
        \end{tabular} \\ \midrule
        
        NES attack & 
        \begin{tabular}[c]{@{}l@{}}
            $\epsilon = 0.1$ \\
            num\_samples = 50 \\
            num\_iterations = 300 \\
            $\sigma = 0.01$ \\
            $\alpha = 0.03$
        \end{tabular} \\ \midrule
        
        Square attack & 
        \begin{tabular}[c]{@{}l@{}}
            $\epsilon = 0.1$ \\
            num\_queries = 5000 \\
            $p\_init = 0.8$
        \end{tabular} \\ \midrule
        
        NP attack & 
        \begin{tabular}[c]{@{}l@{}}
            $\epsilon = 0.05$ \\
            num\_iterations = 1000 \\
            learning\_rate = 0.01
        \end{tabular} \\ \midrule
        
        MCG & 
        \begin{tabular}[c]{@{}l@{}}
            down\_sample\_x = 1 \\
            down\_sample\_y = 1 \\
            finetune\_grow = True \\
            finetune\_reload = True \\
            finetune\_perturbation = True
        \end{tabular} \\ 
        \bottomrule
    \end{tabular}
\end{table}

\end{document}